\documentclass[conference]{IEEEtran}
\IEEEoverridecommandlockouts
\usepackage{cite}
\usepackage{amsmath,amssymb,amsfonts}
\usepackage{graphicx}
\usepackage{textcomp}
\usepackage{xcolor}
\usepackage{hyperref}
\usepackage{amsthm}
\usepackage{algorithm}
\usepackage{booktabs}
\usepackage{wrapfig}
\usepackage[noend]{algorithmic}
\usepackage{subfig}   

\newtheorem{theorem}{Theorem}

\newtheorem{assumption}[theorem]{Assumption}

\def\BibTeX{{\rm B\kern-.05em{\sc i\kern-.025em b}\kern-.08em
    T\kern-.1667em\lower.7ex\hbox{E}\kern-.125emX}}
    
\begin{document}

\title{SGFusion: Stochastic Geographic Gradient Fusion in Federated Learning}



\author{%
  Khoa Nguyen$^{\S*}$,
  Khang Tran$^{\S*}$,
  NhatHai Phan$^{\S}$,
  Cristian Borcea$^{\S}$\\
  Ruoming Jin$^{\P}$,
  Issa Khalil$^{\Im}$\\
\small
  $^\S$New Jersey Institute of Technology, Newark, NJ, USA, 
  $^\P$Kent State University, Kent, OH, USA\\
  $^\Im$Qatar Computing Research Institute, HBKU, Doha, Qatar\\
  E-mail: \{nk569, kt36, phan, borcea\}@njit.edu, rjin1@kent.edu, ikhalil@hbku.edu.qa \\
  $^*$Co-first authors.\\
  \vspace{-3em}
}

\maketitle

\begin{abstract}
This paper proposes Stochastic Geographic Gradient Fusion (\textbf{SGFusion}), a novel training algorithm to leverage the geographic information of mobile users in Federated Learning (FL). SGFusion maps the data collected by mobile devices onto geographical zones and trains one FL model per zone, which adapts well to the data and behaviors of users in that zone. SGFusion models the local data-based correlation among geographical zones as a hierarchical random graph (HRG) optimized by Markov Chain Monte Carlo sampling. At each training step, every zone fuses its local gradient with gradients derived from a small set of other zones sampled from the HRG. This approach enables knowledge fusion and sharing among geographical zones in a probabilistic and stochastic gradient fusion process with self-attention weights, such that \textit{``more similar''} zones have \textit{``higher probabilities''} of sharing gradients with \textit{``larger attention weights.''} SGFusion remarkably improves model utility without introducing undue computational cost. Extensive theoretical and empirical results using a heart-rate prediction dataset collected across 6 countries show that models trained with SGFusion converge with upper-bounded expected errors and significantly improve utility in all countries compared to existing approaches without notable cost in system scalability.
\end{abstract}


\section{Introduction}
Although Federated learning (FL) \cite{mcmahan2017communication} has many applications for mobile users \cite{Zhang2022IoT}, we still need to find a practical solution that obtains good model accuracy while adapting to user mobility behavior, scales well as the number of users increases, and protects user data privacy, which is especially important for mobile sensing applications (e.g., mobile health). 
A common approach toward this goal is to group users into different clusters, each of which is trained in an FL manner to achieve better model performance \cite{NEURIPS2020_ifca, ruan2022fedsoft, long2023multi}. There are different clustering criteria, such as using the objective function on the users' local data distribution \cite{NEURIPS2020_ifca, qu2022convergence}, the gradient similarity \cite{long2023multi, ruan2022fedsoft} or the users' geographical location \cite{jiang2023zone} to reduce the discrepancy of model utility among users and clusters. Among these approaches, leveraging users' geographical information is the most suitable approach to divide the physical space into geographical zones of users mapped to a mobile-edge-cloud FL architecture \cite{jiang2023zone} since it scales well with the increasing number of users in real-world settings. 
By enabling geographical zones of (mobile) users to share gradients with their geographical neighboring (adjacent) zones, FL has shown outstanding performance in model utility and server scalability in real-world deployments of mobile sensing applications, e.g., heart rate prediction and human activity recognition.

\textbf{Challenges.} Although leveraging users' geographical information is promising for real-life FL deployments on mobile devices, the trade-off between model utility and system scalability has yet to be addressed in a systematic way. Specifically, there is a lack of optimized geographical training algorithms. As a result, geographical zones without sufficient training data or appropriate geographical correlations with other zones often have poor model utility.

Addressing this problem is challenging, given the complex and dynamic correlation among geographical zones. In fact, a specific zone can obtain shared gradients from all other zones. However, it will significantly increase computational complexity on all users and edge devices managing geographical zones while offering negligible model utility improvements.
The trade-off between model utility and system scalability raises a fundamental question: \textit{``How to share gradients among geographical zones to optimize model utility without affecting system scalability?''} 

\textbf{Contributions.} To systematically answer this question, we propose Stochastic Geographic Gradient Fusion (\textbf{SGFusion}), a novel FL training algorithm for mobile users. SGFusion models the local data-based correlation among geographical zones of users as a hierarchical random graph (HRG) \cite{clauset2006structural}, optimized by Markov Chain Monte Carlo sampling. HRG represents the probability for one zone to share its gradient with another zone. At a training step, every zone can fuse its local gradients with gradients derived from a set of other zones sampled from the HRG. This approach enables knowledge fusion and sharing among geographical zones in a probabilistic and stochastic gradient fusion process with dynamic attention weights, such that \textit{``more similar''} zones have \textit{``higher probabilities''} of sharing gradients with \textit{``larger attention weights.''} In fact, using HRG reduces and structures the search space, allowing us to identify similar zones better. Zone sampling enables us to reduce the computational cost further. As a result, SGFusion can remarkably improve model utility without introducing undue computational cost. 

Extensive theoretical and empirical results show that models trained with SGFusion converge with upper-bounded expected errors. SGFusion significantly improves model utility without notable cost in system scalability compared with existing approaches. The experiments on heart rate prediction demonstrate that, among the total of 115 zones across 6 countries, \textbf{more than double the number of zones benefit from SGFusion} compared with the state-of-the-art clustering-based FL approaches and their variants \cite{jiang2023zone,wang2023delta}, without slowing the convergence process. SGFusion improves the aggregated model utility across six countries by $3.23\%$ compared with existing approaches. Here is the code for SGFusion: \href{https://anonymous.4open.science/r/SGFusion-BC13/README.md}{\textbf{Code}}.

\section{Background and Related Work}

In FL, a coordination server and a set of $N$ users jointly train a model $h_\theta$, where $\theta$ is a vector of model weights \cite{mcmahan2017communication}. 


\textbf{Clustering-based FL.} Instead of training one global model \cite{mcmahan2017communication}, the service provider in clustering-based FL divides the users into clusters and trains an FL model for each cluster to enhance the model's performance under non-independent and identically distributed (non-IID) data distribution across users \cite{briggs2020federated,ghosh2020efficient,li2023hierarchical,morafah2023flis,Li2022cluster}. The challenges of non-IID data could also be mitigated through new aggregating methods \cite{karimireddy2021scaffoldstochasticcontrolledaveraging,li2020federatedoptimizationheterogeneousnetworks,li2021fedbnfederatedlearningnoniid}. A pioneering work in this setting \cite{briggs2020federated} leverages the hierarchical clustering method to cluster the clients based on their updating gradients. Similarly, Ghosh et al. \cite{ghosh2020efficient} proposed sending multiple models associated with different data distribution to the clients and letting them choose the model that minimizes their data loss. Although these works mitigate the non-IID problems, they incur high computation overhead on the users' devices proportional to the number of clusters, hence limiting the scalability of the existing systems. 

\textbf{Decentralized FL.} Another line of work addressing the scalability problem proposes to replace the centralized coordinating server with multiple servers, each of which being associated with a cluster~\cite{ouyang2023cluster, zhang2021optimizing, long2023multi,wu2023topology,beltran2023decentralized}. For instance, \cite{long2023multi} proposed a multi-center FL setting to personalize FL models to different users' data distribution while avoiding high computation overhead. Similarly, Zhang et al. \cite{zhang2021optimizing} introduced a three-layer collaborative FL architecture with multiple edge servers to learn ML models for IoT devices in FL settings. However, existing decentralized FL approaches usually disregard the users' geographical locations; as a result, they increase communication overhead and do not adapt well to user mobility behaviors in a mobile-edge-cloud FL architecture in practice.

\textbf{Zone FL} is a recent effort to make FL deployments practical in real-world scenarios, which aims at high model utility by adapting well to user mobility behavior while scaling well to the increasing number of users \cite{jiang2023zone}. To achieve this goal, Zone FL divides the physical space into non-overlapping geographical zones and maps the correlation between users, zones, and FL models to a mobile-edge-cloud FL architecture. 


In this setting, an edge-device manages a zone-FL model trained for users whose local data is collected in that zone on their mobile devices. The cloud manages the general zone structure, such as zone granularity. Given a specific zone $z$, the goal is to train an FL model for the zone, $\theta_z$, minimizing an objective function $F_z = \frac{1}{m_z}\sum_{u \in z}\ell(\theta_z, D_u)$, where $\ell(\cdot, \cdot)$ is a loss function (e.g., cross-entropy function) and $m_z$ is the number of users in $z$, using data $D_u$ collected by users $u$ in the zone $z$: $\theta_z^* = \arg\min_{\theta_z}F_z$. Given a set of zones $z \in Z$ where $|Z|$ denotes the number of zones, the goal of Zone FL is to find a set of $\{\theta^*_z\}_{z \in Z}$ that minimizes the average objective function of all the zones, as follows:
\begin{equation}
\{\theta_z^*\}_{z\in Z} = \arg\min_{\{\theta_z\}_{z \in Z}}\frac{1}{|Z|}\sum_{z \in Z} F_z.
\label{Zone-FL Loss}
\end{equation}


To enhance model utility by fostering knowledge fusion among zone models in Eq. \ref{Zone-FL Loss}, the current training algorithm in Zone FL, called Zone Gradient Diffusion (ZGD), enables every zone to fuse its local gradient with gradients derived from its geographical neighboring/adjacent zones. At round $t$, the zone $z$ sends its model weights $\theta_z^t$ to its neighboring zones, denoted as $\mathcal{N}(z)$. Then, the neighboring zones derive their local gradients by using the shared model weight $\theta_z^t$ and their local data, as follows:
\begin{equation}
\forall z' \in \mathcal{N}(z): \nabla_{\theta_z^t}F_{z'} = \frac{1}{m_{z'}}\sum_{u \in z'}\nabla_{\theta_z^t}\ell(\theta_z^t, D_u),
\end{equation}
where $\nabla_{\theta_z^t}\ell(\theta_z^t, D_u)$ is the gradient derived by the user $u$ using the model weight $\theta_z^t$ and their local data $D_u$. 



After receiving all the gradients $\{\nabla_{\theta_z^t}F_{z'}\}$ from neighboring zones $z' \in \mathcal{N}(z)$, the zone $z$ will fuse these gradients associated self-attention coefficients $\{\lambda_{z, z'}\}$ with its local gradient to update its zone model, as follows:
\begin{equation}
\theta_z^{t+1} = {\theta_z^{t} - \eta_t\big[\nabla_{\theta_z^t}F_{z} + \sum_{z' \in \mathcal{N}_z^t} \lambda_{z, z'}\nabla_{\theta_z^t}F_{z'}\big]},
\label{ZGD Formula}
\end{equation}
where $\eta_t$ is the learning rate at round $t$, and the self-attention coefficients capture the normalized similarities between the local gradient of $z$ and the shared gradients from its neighboring zones, as follows: $\lambda_{z,z'} = \frac{\exp{(e_{z, z'})}}{\sum_{\tilde{z} \in \mathcal{N}(z)} \exp{(e_{z, \tilde{z}})}}$, where $e_{z, z'} = \sigma\big(\langle\nabla_{\theta^t_z}F_z; \nabla_{\theta^t_z}F_{z'}\rangle\big)$, $\sigma(\cdot)$ is the sigmoid function and $\langle\cdot;\cdot\rangle$ is the inner product. 

The key idea of Eq. \ref{ZGD Formula} is that the \textit{``more similar''} the gradients of a neighboring zone $z'$ are with those of zone $z$, the \textit{``higher the coefficient''} $\lambda_{z,z'}$ is; thus, resulting in a larger influence of zone $z'$ on the model training of zone $z$.

\section{Stochastic Geographic Gradient Fusion}

Although Zone FL is better at addressing the trade-off between model utility and system scalability compared with classical FL, there is still a fundamental question that it did not address: \textit{``With which zones should a zone $z$ fuse its local gradients at a given training round to achieve high model utility without undue computational cost, and how?''}
Answering this question is non-trivial.  First, fusing knowledge via shared gradients from neighboring zones does not necessarily improve the model utility of a specific zone, given diverse local data distributions among these zones. It is well-known that diverse local data distributions can cause scattered gradients, thus degenerating the FL model utility \cite{Li2020FLChallenges}. Second, a deterministic gradient descent (GD) fusion approach, which uses fixed neighboring zones across training rounds, is not well-optimized since it does not consider the correlation between zone $z$ and all the other zones. A naive solution to this problem is to fuse a zone $z$'s local gradients with the gradients from all geographical zones at each training round. However, applying this (deterministic) GD significantly increases the training cost of $|Z|$ zone-FL models after $T$ training rounds by $|Z|^2 \times N$ times compared with classical FL training \cite{mcmahan2017communication}, where $N$ is the total number of users in all the zones, and $\frac{|Z|^2 \times N}{\sum_{z \in Z}\sum_{z' \in \mathcal{N}(z)}m_z'}$ times compared with ZGD \cite{jiang2023zone}. Therefore, existing training algorithms either degrade the model utility or affect the system scalability.

\textbf{SGFusion Overview.} To address these problems, we propose stochastic geographic gradient fusion (SGFusion), a novel FL training algorithm for mobile users. SGFusion uses a hierarchical random graph (HRG) \cite{clauset2006structural} to model the correlations among zones as sampling probabilities based on the distances between their local data distributions. HRG is scalable due to its ability to efficiently represent statistical correlations among zones when the number of zones increases, making it efficient for settings with large numbers of users.
Then, SGFusion optimizes the HRG by Markov Chain Monte Carlo (MCMC) sampling \cite{newman1999monte}. At each training round, each zone $z$ samples a small set of zones given the HRG to fuse its local gradient with local gradients derived from these zones. This method enables knowledge fusion and sharing among zones, such that \textit{``more similar''} zones have \textit{``higher probabilities''} of sharing gradients with \textit{``larger attention weights''} at each training round of each zone-FL model.
As a result, SGFusion reduces the data diversity and scattered shared gradients while providing sufficient knowledge fused from other zones to improve zone models $\{\theta^*_z\}_{z \in Z}$ in Eq. \ref{Zone-FL Loss}.

\begin{figure}[t]
    \centering  
    \includegraphics[scale=0.37]{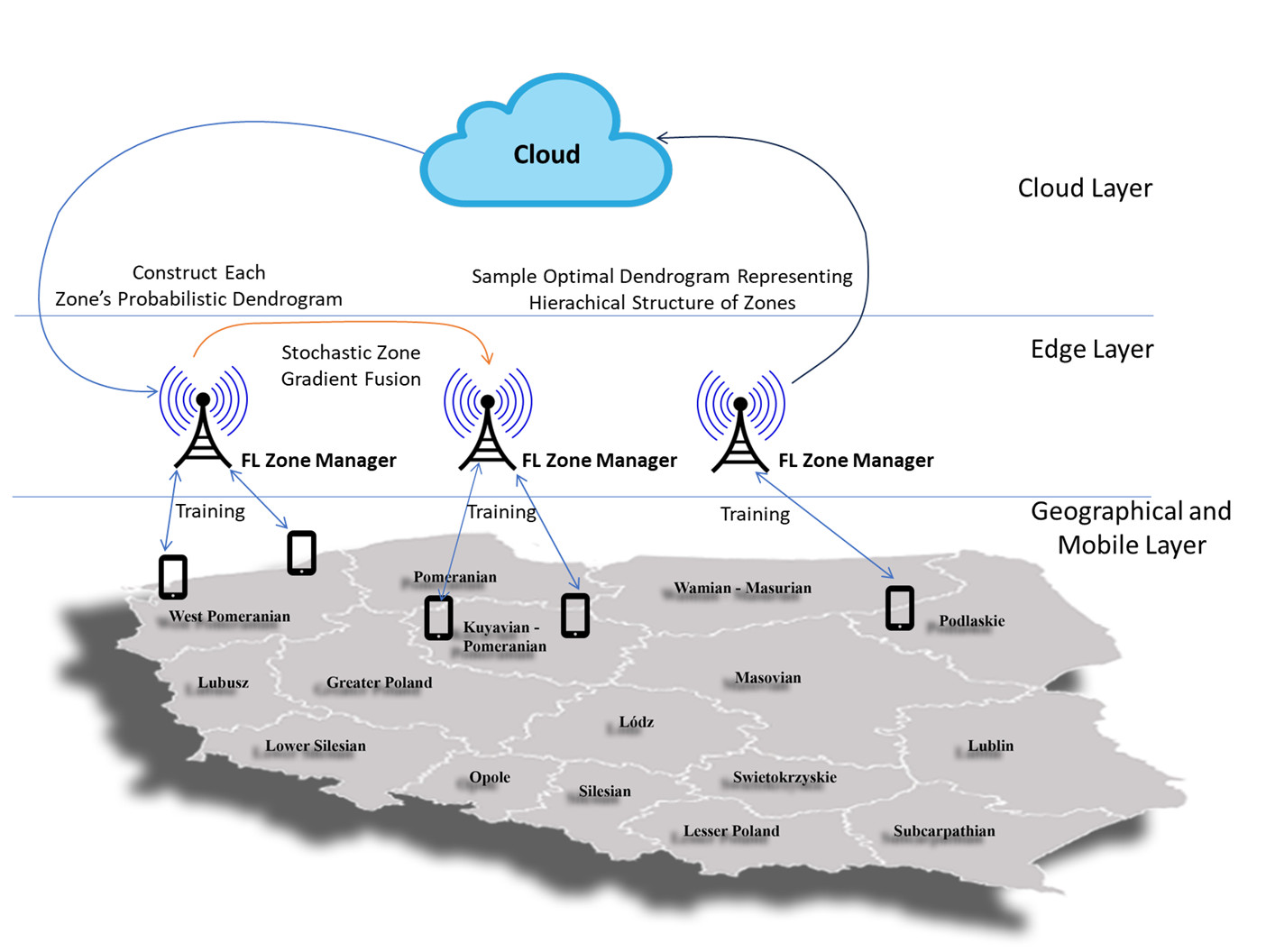} 
    \caption{SGFusion with geographical zones.} 
    \label{fig:sgf-geofl}
\end{figure}

\begin{algorithm}[t]
\caption{Geographic HRG}\label{alg: DenSampling}
\textbf{Input:} Fully connected graph $G$ \\
\text{{\bf Output:} Probabilistic dendrograms}

\begin{algorithmic}[1]
    \STATE \textbf{Dendrogram Sampling and Optimizing Process: }
    \STATE Randomly initialize the dendrogram $\mathcal{T}$ 
    \WHILE{not convergence of $\mathcal{T}$}
    \STATE \textbf{Randomly select} an internal node $r \in \mathcal{T}$
    \STATE \textbf{Uniformly select} either $\alpha$ or $\beta$-transition given $r$ to create the new dendrogram candidate $\mathcal{T}'$ 
    \STATE \textbf{Sample} the transition $\mathcal{T}\rightarrow \mathcal{T}'$ using Eq. \ref{eq:dendrogramupdate}.
    \ENDWHILE
    \STATE\textbf{Construct the Probabilistic Dendograms: } 
    \STATE \textbf{Retrieve} a set of internal ancestor nodes of $z$, denoted as $S_z$, given $\mathcal{T}$
    \STATE \textbf{Initialize} $\mathcal{T}_z$ given $\mathcal{T}$
    \STATE \textbf{Compute} sampling probabilities $\forall r \in S_z: p_r = \frac{\exp(-d_r)}{\sum\limits_{r \in S_z} \exp(-d_r)}$
    \STATE \textbf{Update} the probability value of internal node $r$ in $\mathcal{T}_z$ with $p_r$
    \STATE \textbf{Return:} $\mathcal{T}_z$
\end{algorithmic}
\end{algorithm}

\subsection{Geographic HRG \& Probabilistic Dendrograms}

Figure \ref{fig:sgf-geofl} illustrates the general framework of SGFusion. Given a set of users $u \in z$, each user $u$ collects data $D_{u} = \{(x, y)\}$ in the zone $z$, where $x$ and $y$ are the input and label of a data sample $(x, y)$, respectively. To build the HRG, each user $u$ independently sends their data label distribution\footnote{The distribution can have different forms, such as histograms or class distributions, depending on the downstream learning tasks. Sending a data label histogram to the edge-device can incur a privacy risk, which can be effectively addressed by using the Laplace mechanism to preserve differential privacy \cite{dwork2014}, as in the \ref{DP-preserving Local Data Label Histogram}, without affecting model utility notably.}, denoted as $\mathcal{Y}_u$, aggregated from all labels $\{y\}$ to their corresponding edge-device managing the zone $z$.
The edge device of zone $z$ averages these local data label distributions to create a zone-level data label distribution $\mathcal{Y}_z$, as follows:
\begin{equation}
\forall z \in Z: \mathcal{Y}_z = \frac{1}{m_z}\sum_{u \in z} \mathcal{Y}_u.
\end{equation}

All the edge devices send their zone-level data label distributions $\{\mathcal{Y}_z\}_{z \in Z}$ to the cloud independently. Now, the cloud can construct a fully connected graph $G$ in a centralized manner, in which a node represents a zone and an edge represents the distance, e.g., Euclidean distance \cite{ONEILL200643}, between two zones $z$ and $z'$ using their zone-level data label distributions $d(\mathcal{Y}_z, \mathcal{Y}_{z'})$. Other distance functions such as Manhattan distance \cite{krause1973taxicab} and Minkowski distance \cite{THEODORIDIS2009701} could be used to calculate the distance between two zones $z$ and $z'$.



Given the graph $G$, Alg. \ref{alg: DenSampling} (Lines 1-6) describes how to construct the Geographic HRG and the dendrograms. The cloud randomly samples a hierarchical structure of the zones as a tree dendrogram $\mathcal{T}(\{r, d_r\})$, consisting of $|Z|$ zones as leaf nodes and a set of internal nodes $r$ associated with average distance scores $d_r$ between their left and right sub-trees, $L_r$ and $R_r$, as follows:
\begin{equation}
    \forall r \in \mathcal{T}: d_r = \frac{\sum_{z \in L_r, z' \in R_r} d(\mathcal{Y}_z, \mathcal{Y}_{z'})}{n_{L_r}n_{R_r}},
\end{equation}
where $n_{L_r}$ and $n_{R_r}$ are the numbers of zones in the left and the right sub-trees of the internal node $r$.

\begin{figure}[t]
    \centering
    \includegraphics[width=0.7\linewidth]{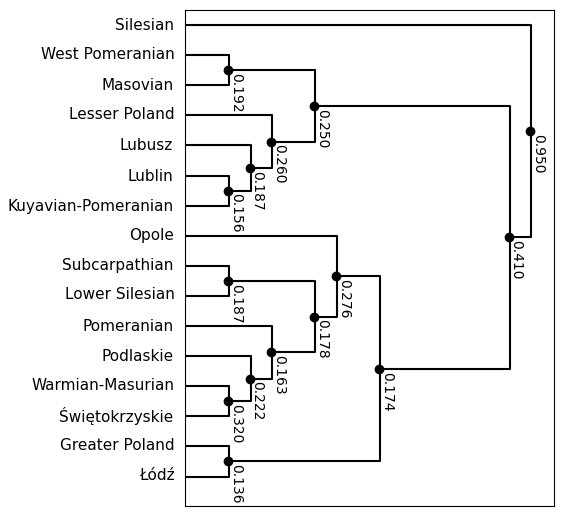}
    \caption{Dendrogram $\mathcal{T}$ with 16 zones in Poland (as shown in Figure \ref{fig:sgf-geofl}) using Euclidean distance.}
    \label{fig:hrg-16z-poland}
\end{figure}

Figure \ref{fig:hrg-16z-poland} illustrates an example of a dendrogram $\mathcal{T}$ with 16 zones as leaf and internal nodes, each of which is associated with a value quantifying the average Euclidean distance between zones from its left and right sub-trees. For example, in the Heart Rate Prediction dataset (HRP) \cite{10.1145/3308558.3313643} that we use in our experiments, $\mathcal{Y}_z$ is the average histogram distribution of the heart rate of all users in the zone $z$. 

To optimize the dendrogram $\mathcal{T}$, SGFusion applies the MCMC sampling to minimize the total average distances in all the internal nodes, indicating that $\mathcal{T}$ can be used to reconstruct the original graph $G$ with minimized loss. In other words, the dendrogram $\mathcal{T}$ is optimized to represent the correlations among zones based on their local data label distributions. The optimization objective of $\mathcal{T}$ is as follows:
\begin{equation}
    \mathcal{T}^* = \arg\min_{\mathcal{T}} \mathcal{L}(\mathcal{T}),
    \label{objectiveHRG}
\end{equation}
where the utility loss of $\mathcal{T}$ is computed by $\mathcal{L}(\mathcal{T}) = \sum\limits_{r\in \mathcal{T}} d_r$.

\begin{figure}[t]            
    \centering
    \includegraphics[scale=0.22]{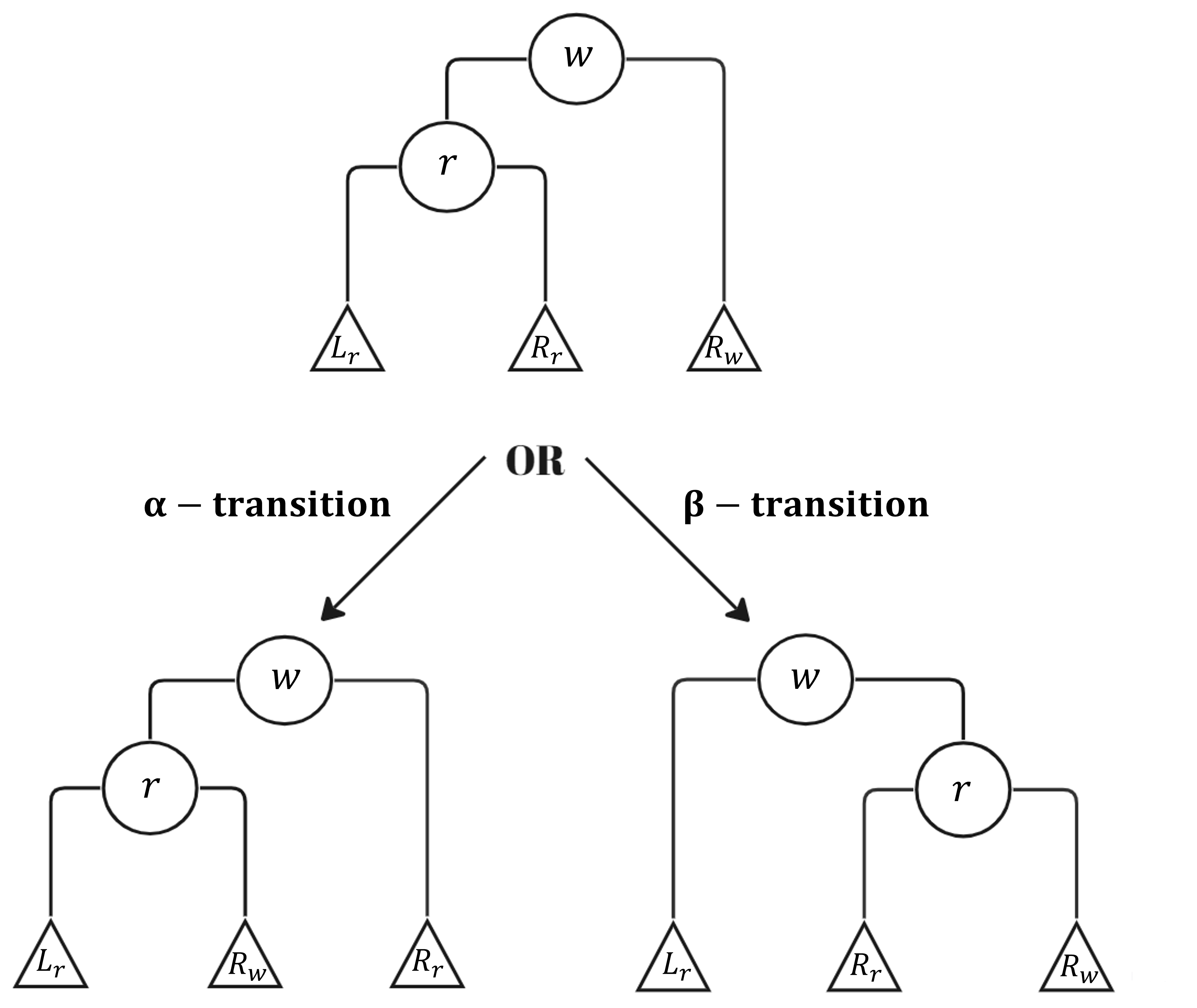}
    \caption{Given a current state of an internal node $r$, there are only two possible candidate states, which are the result of $\alpha$-transition and $\beta$-transition on the node $r$.}
    \label{fig:MCMCSampling}
\end{figure}

At each state, the MCMC sampling picks a random internal node $r$ to re-arrange the structure of its three associated sub-trees, consisting of its children subtrees, i.e., $L_r$ and $R_r$, and a sibling subtree, i.e., $R_w$ (Figure \ref{fig:MCMCSampling}). In the $\alpha$-transition, given an internal node $r$, the order of the three subtrees associated with $r$ is rearranged while keeping the state of $r$. On the other hand, in the $\beta$-transition, the order of the three subtrees does not change, but the state of $r$ is changed compared with $w$. SGFusion uniformly chooses either $\alpha$-transition or $\beta$-transition with an equal probability of 0.5 to create a new dendrogram candidate $\mathcal{T}'$. Then, SGFusion updates $\mathcal{T}$ with a probability $\rho = \min(1, \frac{\exp(\mathcal{L}({\mathcal{T}}))}{\exp (\mathcal{L}({{\mathcal{T'}}}))})$, as follows:


\begin{align}
    \label{eq:dendrogramupdate}
    \mathcal{T} &= \left \{
    \begin{aligned}
        &\mathcal{T'}, &&\text{with probability } \rho  \\
        &\mathcal{T}, &&\text{with probability } 1 - \rho.
    \end{aligned} \right. 
\end{align} \par 


The key idea of updating the dendrogram using Eq. \ref{eq:dendrogramupdate} is that SGFusion always accepts the transition from $\mathcal{T}$ to $\mathcal{T}'$ when it minimizes the utility loss $\mathcal{L}(\mathcal{T})$; otherwise, SGFusion accepts the transition that increases the utility loss $\mathcal{L}(\mathcal{T})$ with the probability of $\frac{\exp(\mathcal{L}({\mathcal{T}}))}{\exp(\mathcal{L}({{\mathcal{T'}}}))}$. SGFusion executes the MCMC sampling process until the convergence of $\mathcal{T}$. 
By doing so, SGFusion optimizes the dendrogram to present the hierarchical structure of all the geographical zones.

\textbf{Probabilistic Dendrogram} (Alg.\ref{alg: DenSampling} Lines 7-12). After building the dendrogram $\mathcal{T}$, the cloud starts constructing a probabilistic dendrogram, denoted as $\mathcal{T}_z$, for each zone $z$, in which the value of an internal node $r$ is the probability $p_r$ of zones in either the left sub-tree or the right sub-tree of $r$ to share gradients with the zone $z$ at a training round.
To construct the probabilistic dendrogram $\mathcal{T}_z$ for a zone $z$, the cloud creates $\mathcal{T}_z$ as a copy of the dendrogram $\mathcal{T}$ with empty values for internal nodes. Then, the cloud computes the probabilities $\{p_r\}$ for ancestors of zone $z$, denoted as $S_z$, by normalizing distance scores $d_r$ of these ancestors:
\begin{equation}
\label{eq:normalize}
\forall r \in S_z: p_r = \exp(-d_r) / \big(\sum\limits_{r \in S_z} \exp(-d_r)\big),
\end{equation}
where $p_r$ is the probability value of an internal node $r$ in $\mathcal{T}_z$ and $\forall z: \sum_{r \in S_z}p_r = 1$. 
For instance, Figure \ref{fig:probaDen} shows the probabilities of zones sharing their local gradients with the zone ``West Pomeranian'' at a training round. Based on the heart rate dataset, zone ``Lublin'' has a probability of $0.293$ to share its local gradient with ``West Pomeranian.''

\begin{figure}[t]
    \centering
    \includegraphics[scale=0.45]{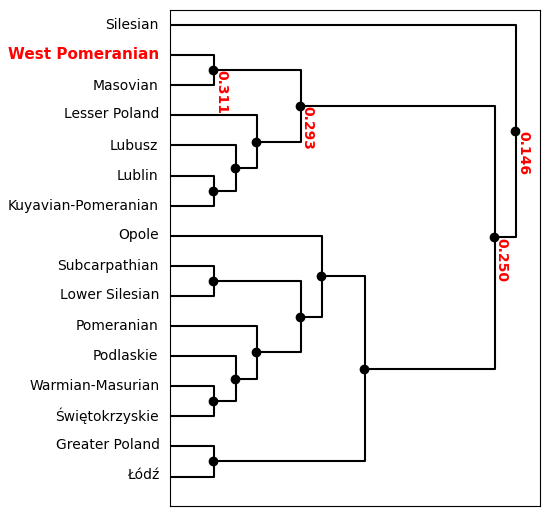}
    \caption{Probabilistic dendrogram of zone ``West Pomeranian'' in Poland derived from Figure \ref{fig:hrg-16z-poland}.}
    \label{fig:probaDen}
\end{figure}

\begin{algorithm}[t]
\caption{SGFusion Training Algorithm}\label{alg: Training}
\hspace*{\algorithmicindent} \textbf{Input: }Zone $z$ and Probabilistic Dendrogram $\mathcal{T}_z$\\
\hspace*{\algorithmicindent} \textbf{Output: }Zone-FL model $\theta_z$
\begin{algorithmic}[1]
    \STATE \textbf{Initialize } the zone's model weight $\theta_z^0$
    \FOR{$t=1,2,\ldots, T$}
    \STATE \textbf{Sample} a set of zones $\mathcal{N}(z,t)$ given $\mathcal{T}_z$
    \STATE $\forall z' \in \mathcal{N}(z,t): e_{z,z'} \leftarrow \sigma\big(\nabla_{\theta^t_z}F_z; \nabla_{\theta^t_z}F_z'\big)$
    \STATE $\forall z' \in \mathcal{N}(z,t): \lambda_{z, z'} \leftarrow \frac{\exp (e_{z,z'})}{\sum_{\tilde{z} \in \mathcal{N}(z,t)} \exp (e_{z,\tilde{z}})}$
    \STATE $\theta_z^{t+1} \leftarrow \theta_z^{t} - \eta_t\big[ \nabla_{\theta^{t}_z}F_z + \sum_{z' \in \mathcal{N}(z, t)}\lambda_{z, z'}\nabla_{\theta^{t}_z}F_{z'}\big]$
    \ENDFOR
    
    \STATE \textbf{Return:} $\theta_z^{T}$
\end{algorithmic}
\end{algorithm}

\subsection{DP-preserving Local Data Label Histogram}

\label{DP-preserving Local Data Label Histogram}

\begin{figure}[t]
    \centering
    \includegraphics[scale=0.25]{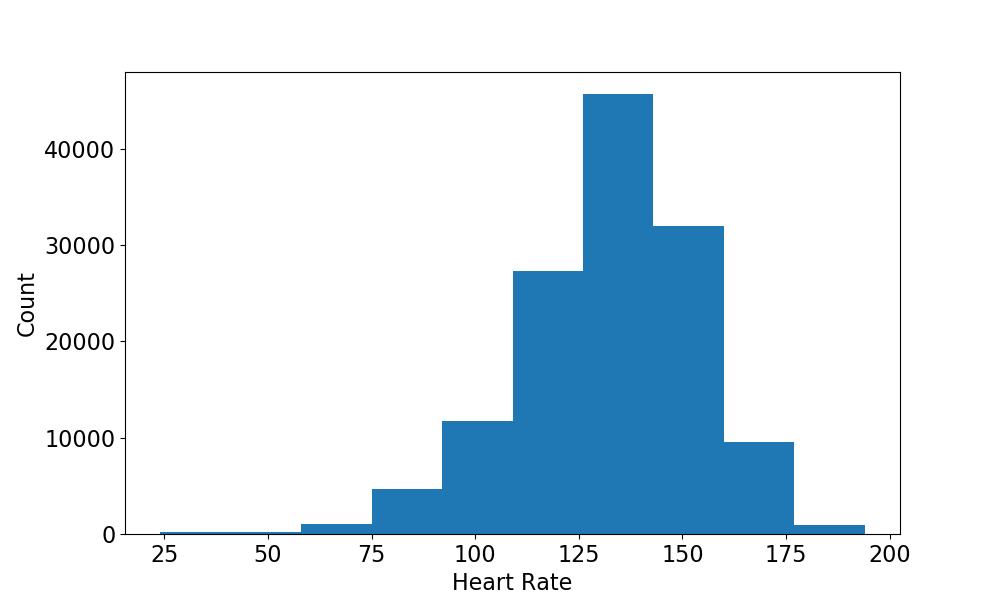}
      \caption{Data label distribution $\mathcal{Y}_u$ of a particular user $u$.}
      \label{fig:DataLabelDistribution}
\end{figure}
Differential privacy (DP) \cite{dwork2014} can be used to release the statistical information of the datasets while protecting individual data privacy. Given a randomized algorithm $\mathcal{M}$, if for any two adjacent input datasets $D$ and $D'$ which differ by only one data point, and for all possible outputs $\mathcal{O}\in Range(\mathcal{M})$, where $Range(\mathcal{M})$ denotes every possible output of $\mathcal{M}$:
\begin{equation}
Pr[\mathcal{M}(D) = \mathcal{O}] \leq e^\epsilon Pr[\mathcal{M}(D') = \mathcal{O}], 
\end{equation}
then, $\mathcal{M}$ satisfies the $\epsilon$-DP. 

The difference in the distribution between $D$ and $D'$ is controlled by the privacy budget $\epsilon$. A smaller $\epsilon$ enforces a stronger privacy guarantee. Due to the risk that labeled data points could leak the user's raw data to the edge devices, SGFusion adapts this mechanism to protect the data label distribution of a user (Figure \ref{fig:DataLabelDistribution}). Given a data label distribution $\mathcal{Y}_u$ of a user $u$, the global sensitivity $\Delta h_u$ of the data label histogram is defined as follows:
\begin{equation}
\Delta h_u = \max_{y} |\mathcal{Y}_u - \mathcal{Y'}_u|,
\end{equation}
where $\mathcal{Y}_u$ and $\mathcal{Y'}_u$ and the data label histograms derived from theQ data $D$ and $D'$, 
which are different by at most one data point's label $y$. Then, we add the Laplace noise $Lap(0,\frac{\Delta h_u}{\epsilon})$ into $\mathcal{Y}_u$ to preserve the DP of the user's data label histogram: 
\begin{equation}
\mathcal{M}(\mathcal{Y}_u, \epsilon) = \mathcal{Y}_u + Lap(0,\frac{\Delta h_u}{\epsilon}).
\end{equation}

Doing so guarantees this mechanism satisfies $\epsilon$-DP following \cite{dwork2014}. We apply this mechanism to independently derive every user's DP-preserving data label histogram. We then use the perturbed histograms to construct the HRG and the probabilistic dendrograms (as in the SGFusion algorithm without privacy protection). Every user $u$ independently sends their DP-preserving data label distribution $\mathcal{M}(\mathcal{Y}_u,\epsilon)$ to their corresponding edge-device managing the zone $z$, and this device aggregates these data label histograms to create a zone-level data label distribution $\mathcal{Y}_z$, as follows:
\begin{equation}
\forall z \in Z: \mathcal{Y}_z = \frac{1}{m_z}\sum_{u\in z}\mathcal{M}(\mathcal{Y}_u,\epsilon).
\end{equation}
Then, all the edge devices send their data label distribution to the cloud independently to construct a fully connected graph $G$, in which a node represents a zone and an edge represents the distance $d(\mathcal{Y}_z, \mathcal{Y}_{z'})$ (e.g., Euclidean distance, Manhattan distance, or Minkowski distance) between two zones $z$ and $z'$. 
Next, with the graph $G$, the cloud constructs the Geographic HRG and the dendrogram based on Alg. \ref{alg: DenSampling} (Lines 1-6). After building the dendrogram, the cloud starts constructing the probabilistic dendrograms (Alg.\ref{alg: DenSampling} Lines 7-12). 

We found that there is a marginal cost to protect the label data. For example, in our experiments, given $\epsilon = 10$, preserving the data label histograms of all users causes 0.15$\%$ difference in Norway's model utility, from 20.04 to 20.07 in the RMSE test loss, while outperforming the other baselines. This cost is low because the aggregated histogram neutralizes the noise added to the users' data label histograms. Therefore, SGFusion still has a good model utility while deriving DP-preserving data label histograms.

\subsection{SGFusion Training}

After constructing all probabilistic dendrograms $\{\mathcal{T}_z\}_{z \in Z}$, the cloud sends them to the corresponding edge servers for training zone models. For brevity, let us describe SGFusion training for a zone $z$, given its associated dendrogram $\mathcal{T}_z$, since the training is done independently for each zone.


At each round $t$, we propose a bottom-up sampling algorithm, as in Figure \ref{fig:NeighboringSampling} and Alg. \ref{alg: Training}, for zone $z$ (a leaf node of the probabilistic dendrogram $\mathcal{T}_z$) to sample a set of zones used for its gradient fusion. From the leaf node represented by zone $z$, our algorithm will travel to every (internal) ancestor node $r$ of $z$, from the closest ancestor node
to the root node, and sample zones in $r$'s sub-trees with the probability $p_r$ associated with the node $r$. Each zone is only sampled once in this process. As a result, zone $z$ identifies a set of zones $\mathcal{N}(z, t)$ at training round $t$ to fuse its local gradients with these sampled zones' local gradients in a self-attention mechanism and update the zone model $\theta_z$, as follows:
\begin{equation}
    \theta_z^{t+1} \leftarrow \theta_z^{t} - \eta_t\big[ \nabla_{\theta^{t}_z}F_z + \sum_{z' \in \mathcal{N}(z,t)
    }\lambda_{z,z'}\nabla_{\theta^{t}_z}F_{z'}\big].
    \label{Aggregating}
\end{equation}
SGFusion trains all zone-FL models $\{\theta_z\}$ using Eq. \ref{Aggregating} independently until the models converge after $T$ training rounds.
\vspace{-10pt}
\section{Theoretical Guarantees}
\subsection{Convergence Analysis}

In this section, we analyze the convergence rate of SGFusion. First, it is worth noting that the HRG is only sampled once before the training process of SGFusion, which is a preprocessing step. Since SGFusion computes the distance between two zones based on Euclidean metric, the complete graph across the zones is an undirected graph, which ensures the HRG sampling process is reversible and ergodic (i.e., any pairs of Dendograms can be transformed to each other with finite sampling steps). Therefore, the HRG sampling process has a unique stationary distribution after it converges to equilibrium \cite{clauset2006structural}, which is reached in our experiments.

\begin{figure}[t]
      \centering
    \includegraphics[scale=0.22]{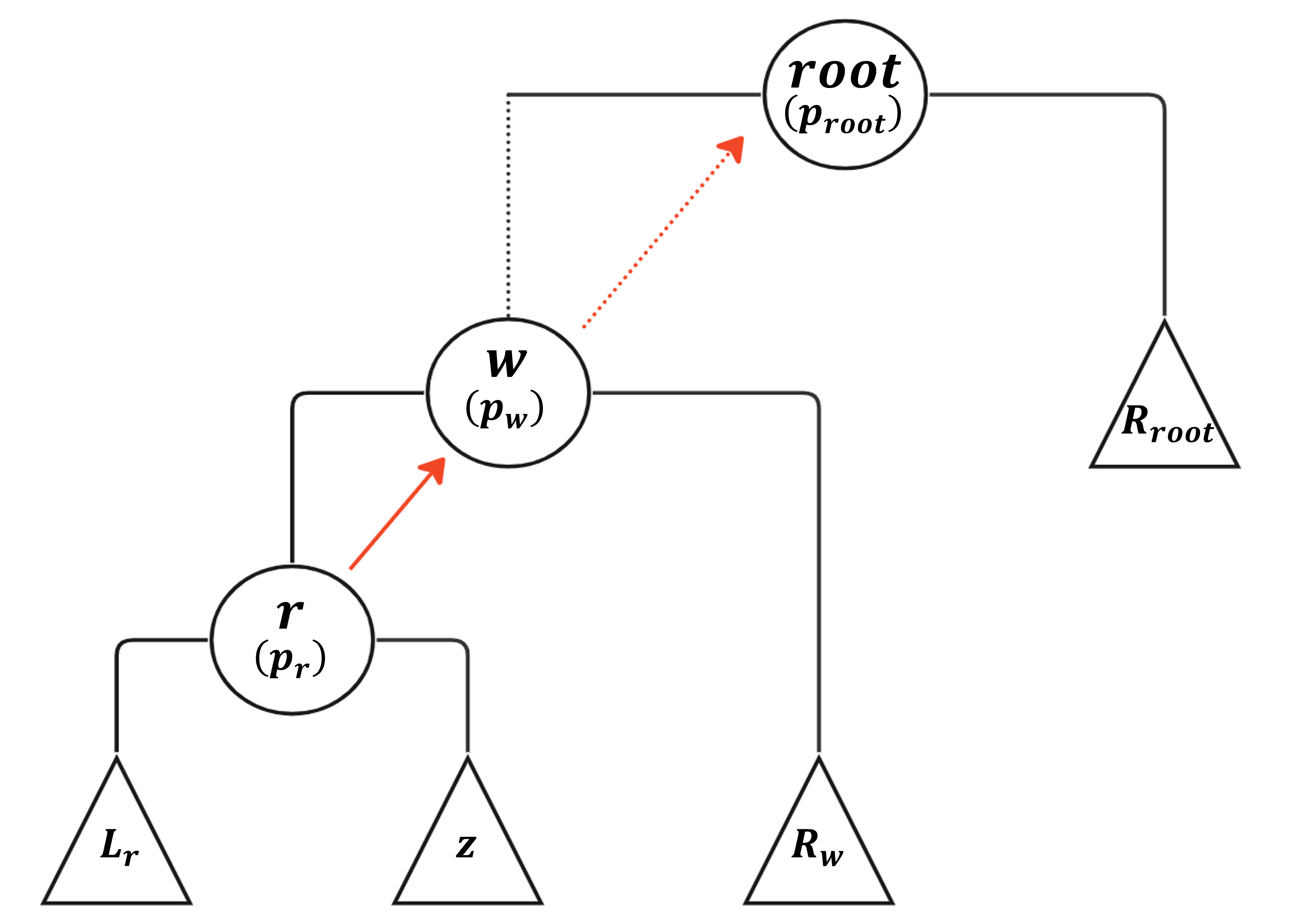}
      \caption{Bottom-up zone sampling algorithm.}
      \label{fig:NeighboringSampling}
\end{figure}

Given a converged HRG representing the relation across the zones, we analyze SGFusion's convergence rate when optimizing a strongly convex and Lipschtiz continuous loss function $F_z, \forall z \in Z$, to provide guidelines for practitioners to employ SGFusion in real-world applications. The key result is that for a particular zone $z$, with a careful learning rate decaying process, SGFusion converges to the global minima of zone $z$ with the rate of $\mathcal{O}(\log(T)/T)$, where $T$ is the number of updating steps. Furthermore, our analysis highlights the impact of the non-IID data property among the zones on the convergence of SGFusion and how SGFusion remedies this impact to enhance the model's utility. To do so, firstly, we consider the following assumptions:

\begin{assumption}\label{assmpt:strong-convex}
    $F_z, \forall z \in Z$ is $\mu$-strongly convex, we have $F_z(\theta') \geq F_z(\theta) + (\theta' - \theta)^\top\nabla_{\theta '} F_z(\theta') + \frac{\mu}{2}\| \theta' - \theta\|^2_2,  \forall \theta ', \theta$.
\end{assumption}
\begin{assumption}
    \label{assmpt:lipschitz}
    $F_z, \forall z \in Z$ is $G$-Lipschitz, such that $\|\nabla_{\theta}F_z(\theta)\|_2 \le G$.
\end{assumption}
\begin{assumption}
    \label{assmpt:bddeviation}
    There exists a constant $\tau$ such that $\|\theta^*_z - \theta^*_{z'}\|_2 \le \tau, \forall z, z' \in Z$, where $\theta^*_z$ is the parameter at the global minimum for zone $z$.
\end{assumption}

These assumptions are typical in providing convergence analysis for FL algorithms in the previous works \cite{li2020convergence,xing2021federated,wu2023faster}. Moreover, they are practically common for many ML models, such as linear regression, logistic regression, and simple neural networks \cite{pilanci2020neural}. Given these assumptions, we can establish the convergence rate of $\theta_z$ through $T$ updating steps, learned by SGFusion, as follows:

\setcounter{theorem}{0}
\begin{theorem}
    \label{theo:conv_theo}
    Let $\theta^T_z$ be the output of Alg. \ref{alg: Training}. If learning rate $\eta_t = \frac{1}{\mu t}$ and Assumption \ref{assmpt:strong-convex} - \ref{assmpt:bddeviation} are satisfied, then the excessive risk $\mathbb{E}[F_z(\theta^T_z)] - F_z(\theta^*_z)$ is bounded by:
    {
    \small
    \begin{multline}
        \mathbb{E}\Big[(F_z(\theta_z^T)\Big] - F_z(\theta^*_z) \le \frac{10G^2}{\mu T} + \frac{16G^2}{\mu T}(1 + \log(\frac{T}{2}))  \\
        + \frac{\bar{G}}{2\mu}\frac{1+\log(T)}{T} + \frac{3G\tau}{2}\sum_{k\neq i}p_{z,z'}, \label{eq:conv_rate}
        \end{multline}
    }where $\bar{G} = G^2\big[1 + 2\sum_{z' \neq z}p_{z,z'} + \sum_{z' \neq z} p_{z',z}(1-p_{z,z'}) + \big(\sum_{z' \neq z}p_{z,z'}\big)^2\big]$, $p_{z,z'}$ is the probability to sample zone $z'$ for the updating process of zone $z$ given the dendrogram $\mathcal{T}$, and the expectation is over the randomness of SGFusion.
\end{theorem}

\begin{proof}
By the updating process of \textsc{SGFusion}, we have that:
{
\small
\begin{align}
    \theta_z^{t+1} = \theta_z^t - \eta_t\Big[\nabla_{\theta_z^t}F_z + \sum_{z' \neq z} a_{z,z'}\lambda_{z,z'}\nabla_{\theta_z^t}F_{z'}\Big] \nonumber,
\end{align}
}where $a_{z,z'} \sim Bernoulli(p_{z,z'})$ 
with $p_{z,z'}$ is the probability to sample zone $z'$ for the updating process of zone $z$. Let us denote $g_z^t = \nabla_{\theta_z^t}F_z + \sum_{z' \neq z} a_{z,z'}\lambda_{z,z'}\nabla_{\theta_z^t}F_{z'}$, we can expand:
{
\small
\begin{align}
    \|\theta_z^{t+1} &- \theta\|_2^2 = \|\theta_z^t - \eta_t g_z^t - \theta\|_2^2 \nonumber \\
    &= \|\theta_z^{t} - \theta\|_2^2 -2\eta_t\langle g_z^t, \theta_z^t - \theta\rangle + \eta^2_t\|g_z^t\|^2_2 \nonumber.
\end{align}
}By re-arranging the equations, we have: 
{
\small
\begin{align}
    \langle g_z^t, \theta_z^t - \theta\rangle &= \frac{1}{2\eta_t}(\|\theta_z^{t} - \theta\|_2^2 - \|\theta_z^{t+1} - \theta\|_2^2) + \frac{\eta_t}{2}\|g_z^t\|^2_2 \nonumber \\
    &= \frac{1}{2\eta_t}A_1 + \frac{\eta_t}{2}A_2.
    \label{eq:derived-1}
\end{align}
}We can bound the expectation of $A_2$ as follows:
{\small
\begin{align}
    &A_2 = \Big\|\nabla_{\theta_z^t}F_z + \sum_{z' \neq z} a_{z,z'}\lambda_{z,z'}\nabla_{\theta_z^t}F_{z'}\Big\|_2^2 \nonumber\\
    &\le \Big[\|\nabla_{\theta_z^t}F_z\|_2 + \sum_{z' \neq z} a_{z,z'}\lambda_{z,z'}\|\nabla_{\theta_z^t}F_{z'}\|_2\Big]^2 \\
    &\le G^2\Big[1 + \sum_{z' \neq z} a_{z,z'}\Big]^2 \nonumber,
\end{align}
}where the second inequality is due to the Assumption \ref{assmpt:lipschitz}, and $\lambda_{z,z'} \le 1, \forall z, z'$. Therefore, in expectation, we have:
{
\footnotesize
\begin{align}
    &\mathbb{E}(A_2) \le G^2\mathbb{E}\Big[1 + \sum_{z' \neq z} a_{z,z'}\Big]^2 \nonumber\\
    &= G^2\Bigg[1 + 2\mathbb{E}\Big(\sum_{z' \neq z} a_{z,z'}\Big) + \mathbb{E}\Big[\Big(\sum_{z' \neq z} a_{z,z'}\Big)^2\Big]\Bigg] \nonumber\\
    &= G^2\Bigg[1 + 2\mathbb{E}\Big(\sum_{z' \neq z} a_{z,z'}\Big) + Var\Big(\sum_{z' \neq z} a_{z,z'}\Big) + \mathbb{E}\Big(\sum_{z' \neq z} a_{z,z'}\Big)^2\Bigg] \nonumber\\
    &= G^2\Bigg[1 + 2\sum_{z' \neq z}p_{z,z'} + \sum_{z' \neq z} p_{z,z'}(1-p_{z,z'}) + \Big(\sum_{z' \neq z'}p_{z,z'}\Big)^2\Bigg] = \bar{G} \nonumber,
\end{align}
}where $p_{z,z'}, \forall z,z'$ are fixed given a HRG. Furthermore, we can observe that:
\begin{align}
    \langle g_z^t, \theta_z^t - \theta\rangle &= \langle \nabla_{\theta_z^t}F_z, \theta_z^t - \theta\rangle \nonumber \\ &+ \sum_{z' \neq z} a_{z, z'}\lambda_{z,z'}\langle \nabla_{\theta_z^t}F_{z'}, \theta_z^t - \theta\rangle \nonumber.
\end{align}By the convexity of $F_z, \forall z$, we can have that:
{
\small
\begin{align}
    &\langle \nabla_{\theta_z^t}F_z, \theta_z^t - \theta\rangle \ge F_z(\theta_z^t) - F_z(\theta) \nonumber\\
    &\langle \nabla_{\theta_z^t}F_{z'}, \theta_z^t - \theta\rangle \ge F_{z'}(\theta_z^t) - F_{z'}(\theta) \nonumber\\
    &\qquad\qquad\qquad= F_{z'}(\theta_z^t) + F_{z'}(\theta_{z'}^*) - F_{z'}(\theta_{z'}^*) - F_{z'}(\theta) \nonumber\\
    &\qquad\qquad\qquad\ge -|F_{z'}(\theta_{z'}^*) - F_{z'}(\theta)| \ge -G\|\theta_{z'}^* - \theta\|_2 \nonumber.
\end{align}
}Therefore, putting to Eq. \eqref{eq:derived-1}, it follows that:
{
\small
\begin{align}
    F_z(\theta_z^t) - F_z(\theta) &\le \frac{1}{2\eta_t}A_1 + \frac{\eta_t}{2}A_2 \nonumber \\
    &+ G\sum_{z'\neq z}a_{z,z'}\lambda_{z,z'}\|\theta_{z'}^* - \theta\|_2 \nonumber.
\end{align}
}Let $v$ be an arbitrary element in $\{1, \dots, \lfloor T/2\rfloor\}$. Summing over all over $t= T-v, \dots, T$, setting $\eta_t = \frac{1}{\mu t}$, and taking expectation on both sides, we have:
{\small
\begin{align}
    &\mathbb{E}\Big[\sum_{t=T-v}^T(F_z(\theta_z^t) - F_z(\theta))\Big] \le \frac{\mu(T-v)}{2}\mathbb{E}(\|\theta_z^{T-v} - \theta\|_2^2) \nonumber\\
    &+ \frac{\mu}{2}\sum_{t=T-v+1}^{T}\mathbb{E}(\|\theta_z^{t} - \theta\|_2^2) + \frac{\bar{G}}{2\mu}\sum_{t=T-v}^{T}\frac{1}{t} \nonumber \\
    &+ G(T-u+1)\mathbb{E}\Big[\sum_{k\neq i}a_{z,z'}\lambda_{z,z'}\|\theta_{z'}^* - \theta\|_2\Big] \nonumber.
\end{align}
}By the analysis from Theorem 1 from Shamir et al. \cite{shamir2013stochastic} for strongly convex function and replace $\theta$ by $\theta_z^*$, we have:
{\small
\begin{align}
    &\mathbb{E}\Big[(F_z(\theta_z^T)\Big] - F_i(\theta^*_z)) \le \frac{10G^2}{\mu T} + \frac{16G^2}{\mu T}(1 + \log(\frac{T}{2}))  \nonumber \\
    &+ \frac{\bar{G}}{2\mu}\frac{1+\log(T)}{T} + \frac{3G}{2}\mathbb{E}\Big[\sum_{k\neq i}a_{z,z'}\lambda{z,z'}\|\theta_{z'}^* - \theta_z^*\|_2\Big] \nonumber.
\end{align}
}By the Assumption \ref{assmpt:bddeviation}, we have:
{\small
\begin{align}
    &\mathbb{E}\Big[(F_z(\theta_z^T)\Big] - F_z(\theta^*_z)) \le \frac{10G^2}{\mu T} + \frac{16G^2}{\mu T}(1 + \log(\frac{T}{2})) \nonumber \\
    &+ \frac{\bar{G}}{2\mu}\frac{1+\log(T)}{T} + \frac{3G\tau}{2}\sum_{k\neq i}p_{z,z'} \nonumber,
\end{align}
}which concludes the proof.    
\end{proof}

As $T \rightarrow \infty$, we can induce from Eq. \eqref{eq:conv_rate} that SGFusion converges to the global minima of each zone $z \in Z$ with the rate of $\mathcal{O}(\log(T)/T)$. From the last term of Eq. \eqref{eq:conv_rate}, we see that even if $T \rightarrow \infty$, the performance of SGFusion is limited by the non-IID property among the zones quantified by $\tau$, since $\tau \rightarrow \infty$ will enlarge the expected excessive risk. This impact of $\tau$ is consistent with the theoretical and empirical results of previous works \cite{li2020convergence}.
However, focusing on the last term of Eq. \eqref{eq:conv_rate}, it also highlights how SGFusion remedies the non-IID problem from the theoretical point of view. Specifically, as $\tau \rightarrow \infty$, which means we have more diverse data distributions among different zones, the $p_{z, z'}, \forall z,z' \in Z$ will decrease due to the normalization in Line 10, Alg. \ref{alg: DenSampling}. Hence, SGFusion decreases last term's value in a heavily non-IID setting, resulting in a better model utility for each zone.


\subsection{Complexity Analysis}
 
This section analyzes the complexity of SGFusion in HRG sampling and SGFusion's training processes. For the HRG sampling, given $Z$ zones, there are $|Z|(|Z|-1)/2$ pairs of zones for the fully connected graph $G$. Thus, the computation complexity to construct the HRG is $O(|Z|^2)$ because we need to compute the distance of each pairs. Then, given $|n_{L_r}|$ and $|n_{R_r}|$ are the numbers of zones in the left and the right subtrees of the dendrogram $\mathcal{T}$, it takes $O(n_{L_r}*n_{R_r})$ to compute the utility loss $\mathcal{L}(\mathcal{T})$ for one MCMC step. By applying the Cauchy-Schwarz inequality, we have that:

\begin{equation}
 n_{L_r}n_{R_r} \leq \frac{(n_{L_r}+n_{R_r})^2}{4} \leq \frac{|Z|^2}{4}.
\end{equation}

Therefore, it requires $O(M|Z|^2)$ in the worst case for the dendrogram $\mathcal{T}$ to converge, where $M$ is the number of MCMC steps until the convergence of $\mathcal{T}$. Then, to construct the probabilistic (binary) dendrograms, it takes $O(\log |Z|)$ to traverse through the depth of the dendrogram. Hence, it takes $O(|Z|\log |Z|)$ to get the $p_r$ for all the zones. Therefore, the computation complexity of the HRG sampling process scales by $O(|Z|^2+ M|Z|^2 + |Z|\log|Z|) \approx O(M|Z|^2)$. However, this process is only executed once before the SGFusion training. Furthermore, the number of zones $|Z|$ is generally small (less than $40$), resulting in low computation cost.

Regarding to the SGFusion training process, we can compute the gradient update for each zone in parallel. Therefore, we consider the computational complexity of a single zone $z$ as follows. In a training round $t$, the computational complexity to compute the gradient update for the zone's model $\theta_z$ is $O(\mathcal{N}(z,t)) \approx O(|Z|)$, where $\mathcal{N}(z,t)$ is the number of neighboring zones which share the gradients with the zone $z$. Furthermore, the training process is executed with $T$ updating steps. Thus, the computational complexity training process is scaled by $O(T|Z|)$, which is linear to the number of zones $|Z|$. As a result, SGFusion remains computationally scalable in the real-world scenario.

\section{Experiments and Evaluation}
\label{Experiments and Evaluation}

We conduct extensive experiments using a real-world dataset collected across six countries to evaluate the performance of SGFusion in comparison with state-of-the-art baselines, focusing on the following aspects: \textbf{(1)} Assessing zone-FL model utility enhanced by SGFusion; \textbf{(2)} Evaluating system scalability of SGFusion through its convergence rate; 
and \textbf{(3)} Understanding the contribution of each component in SGFusion to the overall utility. 


\begin{table}[t] 
\centering 
\caption{Breakdown of the HRP dataset for top six countries: Norway, Spain, US, Thailand, France, and Poland.}
\label{tb:dataBreakdown}   
\resizebox{.49\textwidth}{!}{
\begin{tabular}{ccccc}
\toprule
& \# users & \# samples & \# zones & avg \# samples/zone\\ 
\midrule
Norway & 48 & 5,902 & 13 & 454.00 \\ 
Spain & 110 & 9,609 & 13 & 739.15 \\ 
US & 99 & 14,774 & 32 & 461.69 \\ 
Thailand & 105 & 10,970 & 23 & 476.96 \\ 
France & 67 & 8,094 & 18 & 449.67 \\ 
Poland & 205 & 16,907 & 16 & 461.69 \\ 
\bottomrule
\end{tabular}}
\end{table}


\textbf{Datasets.} We use the heart rate prediction (HRP) dataset \cite{10.1145/3308558.3313643}, consisting of approximately $168,000$ workout records of $956$ users collected from $33$ countries to evaluate SGFusion. Similar to \cite{jiang2023zone}, we select the top $6$ countries with more than $10$ zones having good numbers of data samples, i.e., more than $450$ data samples per zone on average, in our experiment. HRP is an outstanding dataset for our study, and it has sufficient users and geographical information across multiple countries. The availability of such datasets is limited in the real world.


\textbf{Models and Metrics.} We leverage a Long Short-Term Memory (LSTM) model  \cite{10.1145/3308558.3313643} to forecast the heart rate from the input features, such as workout altitude, distance, and elapsed time (or speed). We use the Root Mean Square Error (RMSE) as the main metric to evaluate the model utility. The lower the value of RMSE, the better the model utility. 

\textbf{Established Baselines.} We consider a variety of baselines: \textbf{(1)} The classical \textbf{FedAvg} \cite{mcmahan2017communication}; \textbf{(2)} Geographical FL approaches, including Static Geographical FL (\textbf{SGeoFL}) and Deterministic Zone Gradient Diffusion (\textbf{D-ZGD}) \cite{jiang2023zone}; \textbf{(3)} \textbf{IFCA}, an iterative federated clustering algorithm \cite{NEURIPS2020_ifca}; \textbf{(4)} A multi-center FL approach Stochastic Expectation Maximization FL (\textbf{FedSEM}) \cite{long2023multi}; and \textbf{(5)} A sampling FL mechanism \textbf{FedDELTA}  \cite{wang2023delta} and its variant \textbf{DELTA-Z} proposed to adapt DELTA to geographical zones. FedAvg is the traditional FL setting, where all users jointly train a global FL model. In SGeoFL, users are geographically separated into zones, and every zone trains its own zone-FL model independently without gradient fusing with other zones. D-ZGD is the state-of-the-art training algorithm with self-attention following Eq. \ref{ZGD Formula} given geographical zones. IFCA is a clustering-based FL in which cluster identities of user and model parameters are optimized via a gradient descent process. We apply IFCA on top of each country, where users are distributed and partitioned into clusters without considering any geographical location. FedSEM utilizes an expectation-maximization framework to optimize the client clusters during the FL process. We also adapt FedSEM to the countries' models. 
DELTA is an FL sampling mechanism in which users are selected at each training round to reduce the diversity of their local gradients and variance, improving the learning process. FedDELTA is an FL approach that utilizes a state-of-the-art sampling mechanism, DELTA, to boost the model performance. Also, since DELTA is not originally designed for geographical zones, we adapt it to DELTA-Z by letting DELTA treat the zone $z$'s model as a central model jointly trained by all users from sampled zones and the zone $z$. For a fair comparison, we assign a similar number of zones with D-ZGD for the zone sampling process of DELTA-Z.  

\begin{table}[t] 
\centering 
\caption{SGFusion vs. D-ZGD and DELTA-Z regarding the number of zones with higher model utility.}
\label{tb:SGFusionVsD-ZGD and DELTA-Z} 
\resizebox{0.49\textwidth}{!}{
\begin{tabular}{cccc}
    \toprule
    & D-ZGD & SGFusion & (\%) Gain\\ 
    \midrule
    Norway & 5 & 8 & 60.00\% \\
    Spain & 5 & 8 & 60.00\% \\
    US & 10 & 22 & 120.00\% \\
    Thailand & 8 & 15 & 87.50\% \\
    France & 6 & 12 & 100\% \\ 
    Poland & 4 & 12 & 200\% \\ 
    \bottomrule
\end{tabular} \hfill
\begin{tabular}{ccc}
\toprule
DELTA-Z & SGFusion & (\%) Gain\\ 
\midrule
4 & 9 & 125.00\% \\ 
4 & 9 & 125.00\% \\ 
8 & 24 & 200.00\% \\ 
11 & 12 & 9.09\% \\ 
7 & 11 & 57.14\% \\ 
4 & 12 & 200.00\% \\
\bottomrule
\end{tabular}} 
\end{table}

\textbf{Variants of SGFusion.} We consider a variant of SGFusion, called $\chi$-\textbf{SGFusion}, in which every zone $z \in Z$ samples the same number of zones $\chi$ in their gradient fusion across training rounds. Also, we include \textbf{top-$k$-SGFusion}, in which every zone $z$ uses top-$k$ most similar zones $\{z'\}$, i.e., smallest distance $d(\mathcal{Y}_z, \mathcal{Y}_{z'})$, in its gradient fusion across training rounds. The goals of considering these variants of SGFusion are: \textbf{(1)} For a fair comparison with D-ZGD, we set $\chi$ for every zone $z$ to be the same with the number of neighboring zones of the zone $z$ used in D-ZGD; and \textbf{(2)} Evaluating the effect of stochastic gradient fusion compared with deterministic gradient fusion using either a fixed number of sampled zones or top-$k$ most similar zones with different values of $k$.

\begin{table}[!t] 
\centering 
\caption{$\chi$-SGFusion vs. D-ZGD and DELTA-Z regarding the number of zones with higher model utility.}
\label{tb:chiSGDFusionVsD-ZGD and DELTA-Z}
\resizebox{0.49\textwidth}{!}{
\begin{tabular}{cccc}
\toprule
& D-ZGD & $\chi$-SGFusion & (\%) Gain \\
\midrule
Norway & 6 & 7 & 16.67\% \\
Spain & 5 & 8 & 60.00\% \\
US & 11 & 21 & 90.91\% \\
Thailand & 10 & 13 & 30.00\% \\
France & 4 & 14 & 250.00\% \\
Poland & 7 & 9 & 28.57\% \\
\bottomrule
\end{tabular} \hfill
\begin{tabular}{ccc}
\toprule
DELTA-Z  & $\chi$-SGFusion & (\%) Gain \\ 
\midrule
4 & 9 & 125.00\% \\
3 & 10 & 233.33\% \\
7 & 11 & 57.14\% \\
10 & 13 & 30.00\% \\
7 & 11 & 57.14\% \\
8 & 8 & 0.00\% \\ 
\bottomrule
\end{tabular}}
\end{table}

\begin{figure}[t]
      \centering
    \includegraphics[scale=0.34]{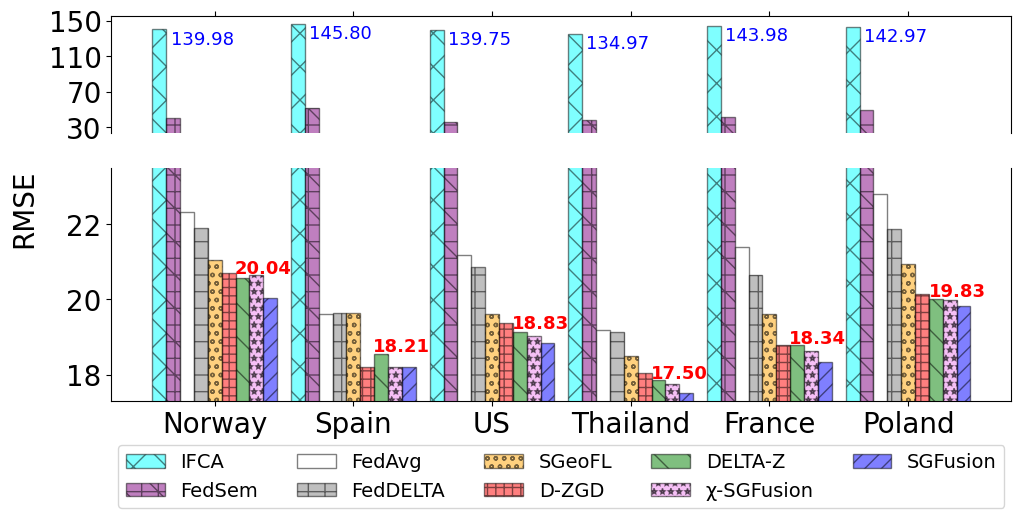}
      \caption{Model utility across countries (smaller, the better).}
      \label{fig:ModelUtility}
      \vspace{-15pt}
\end{figure}

\subsection{Utility and Scalability}

\textbf{Zone and Country Level Utility.} Among all the baselines, D-ZGD and DELTA-Z achieve the best performance (Figure \ref{fig:ModelUtility}). 
Tables ~\ref{tb:SGFusionVsD-ZGD and DELTA-Z} and~\ref{tb:chiSGDFusionVsD-ZGD and DELTA-Z} present the utility of SGFusion compared with these best baselines at the zone-level FL model. SGFusion and $\chi$-SGFusion achieve significantly better model utility on most zone-FL models than D-ZGD and DELTA-Z across all countries. For instance, in Poland, SGFusion achieves better model utility in 12 zones compared with 4 zones of D-ZGD, registering a 200\% improvement. Similar results are observed from other countries and with other baselines, indicating the sharp enhancement of the model utility by the SGFusion at the level of zone-FL models. Across 115 zones in six countries, more than double the number of zones benefit from SGFusion, i.e., 77 zones, than the best baselines, i.e., 38 zones.

Although the improvement at the country level (Figure \ref{fig:ModelUtility}) is not as clear as the zone level since each country uses an aggregated model from its zone models, the model utility at the country level can strengthen our observation. SGFusion improves the aggregated model utility across six countries by $3.23\%$ on average compared with existing approaches. Note that $\chi$-SGFusion, which uses the same numbers of sampled zones with D-ZGD, also outperforms D-ZGD across six countries and outperforms DELTA-Z in five countries with a comparable result in Norway.
\vspace{5pt}






\textbf{Convergence Speed.} As shown in Figure \ref{fig:Learning curve}, SGFusion and $\chi$-SGFusion have a similar convergence speed compared with D-ZGD. The key reason for this result is that SGFusion utilizes relatively small numbers of sampled zones to train a specific zone-FL model on average compared with D-ZGD and DELTA-Z (Figure \ref{fig:Average number of sampled zones}). In fact, at a training round, the average numbers of sampled zones in SGFusion are smaller than the ones in D-ZGD and DELTA-Z across 4 over 6 countries and are comparable in the remaining countries. 
To shed light on why using smaller sampled zones for gradient fusion in SGFusion enables us to avoid negative impacts on convergence speed and system scalability, while still resulting in better model utility, we conduct a \textbf{homophily data analysis} on these sampled zones. This experiment studies the average homophily \cite{khanam2023homophily} across the zones $z \in Z$ and across $T$ updating steps, quantified as: 
{\small
\begin{align}
\frac{1}{T}\sum_{t=1}^T\Big[\frac{1}{|Z|} \sum_{z \in Z} \Big(\frac{1}{|\mathcal{N}(z,t)|}\sum_{z' \in \mathcal{N}(z, t)} d(\mathcal{Y}_z, \mathcal{Y}_{z'})\Big)\Big],
\end{align}
}

\begin{figure}[t]
    \centering
    \subfloat[Learning curve.]{\includegraphics[width=0.45\columnwidth]{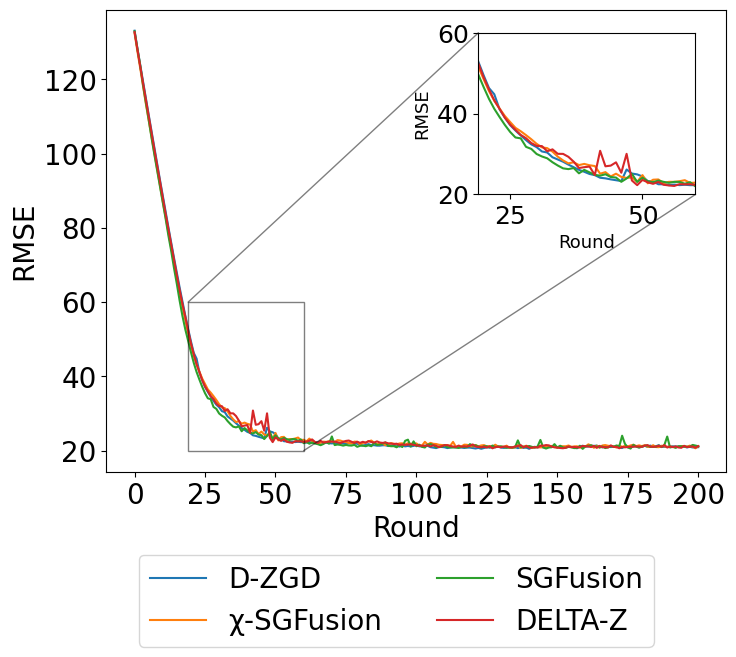}\label{fig:Learning curve}}
    \subfloat[Average number of sampled zones.]
    {\includegraphics[width=0.56\columnwidth]{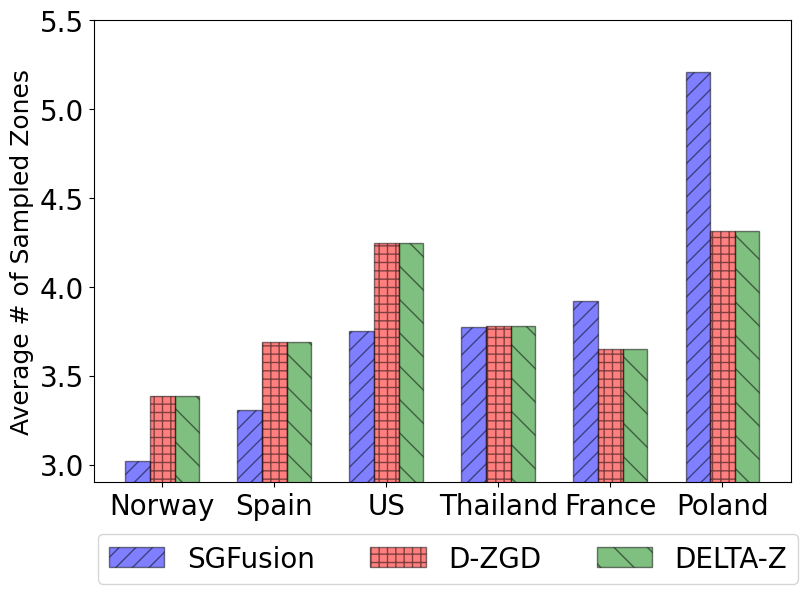}\label{fig:Average number of sampled zones}}
    \caption{Learning curves of the country-level FL models using data collected from Norway (best view in color).}
    \label{fig:LearningCurveandAvergeofSampledZones}
    \vspace{-15pt}
\end{figure}
 
\noindent where $d(\cdot, \cdot)$ is the metric measuring the data label distributions between zones $z$ and $z'$, e.g., Euclidean distance \cite{ONEILL200643}. Intuitively, the lower average homophily, the more similar the label distribution of a zone $z$ and the label distribution of its sampled zones.
\begin{wrapfigure}{r}{0.25\textwidth}
    \begin{center}
\includegraphics[width=0.4\columnwidth]{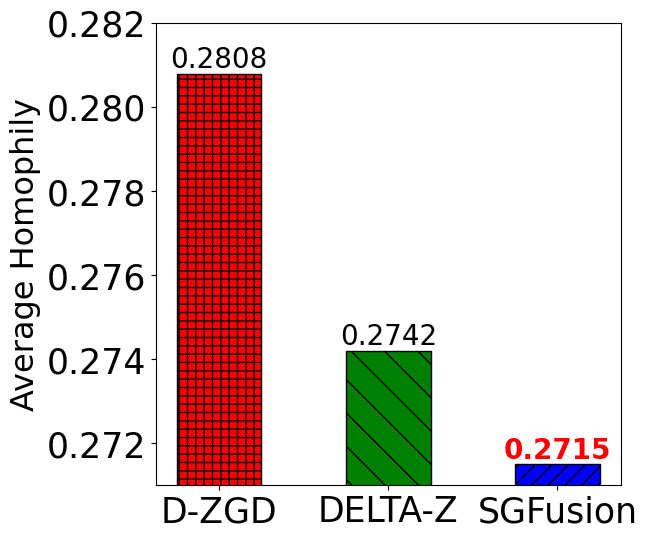}
        \caption{Average homophily of zones across six countries.}
        \label{fig:Homophily}
    \end{center}
\end{wrapfigure}
SGFusion achieves a lower value of average homophily across six countries compared with DELTA-Z and D-ZGD (Figure \ref{fig:Homophily}). \textit{The results highlight that SGFusion obtains a better set of sampled zones for gradient fusion, one key property contributing to the improvement in the model utility of SGFusion without affecting system scalability, demonstrated through its convergence speed.}

\subsection{Stochastic vs. Deterministic} 

\begin{figure}[t]
      \centering
    \includegraphics[scale=0.27]{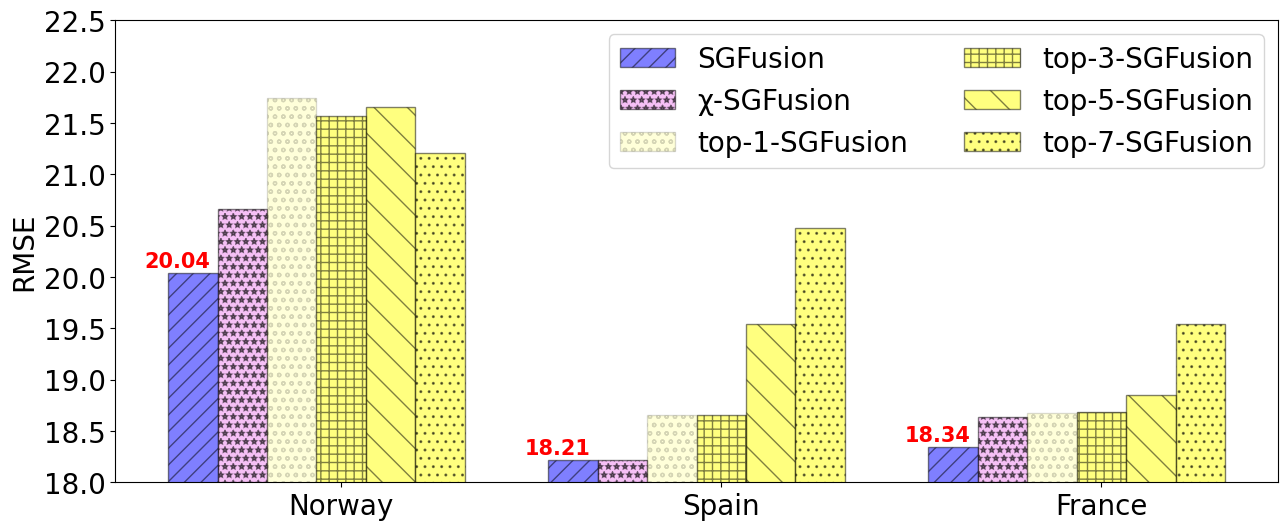} 
      \caption{Deterministic vs. stochastic gradient fusion.}
      \label{fig:StochasticVsDeterministic}
\end{figure}

In this experiment, we investigate the benefit of stochastic gradient fusion by comparing SGFusion with the top-$k$-SGFusion where $k$ varies from 1 to 7, and with $\chi$-SGFusion. The top-$k$-SGFusion is a deterministic gradient fusion approach using the top-$k$ most similar zones to fuse gradients. Meanwhile, $\chi$-SGFusion is a partially stochastic gradient fusion algorithm, in which the number of sampled zones for every zone $z$ is deterministic while the sampled zones can change across training rounds. Therefore, using the top-$k$-SGFusion and $\chi$-SGFusion offers a comprehensive evaluation of the stochastic gradient fusion effect in SGFusion.

In Figure \ref{fig:StochasticVsDeterministic}, SGFusion notably outperforms the top-$k$-SGFusion with $k \in [1, 7]$ and $\chi$-SGFusion. The reason is that using either deterministic values of top-$k$ or deterministic numbers of sampled zones does not offer a good balance between obtaining sufficient knowledge fused from sampled zones and mitigating the increase of discrepancy among fused gradients when the number of sampled zones increases. The stochastic zone sampling approach remedies this problem by enabling knowledge fusion and sharing among zones, such that \textit{``more similar''} zones have \textit{``higher probabilities''} of sharing gradients with \textit{``larger attention weights''} at a training round of a zone-FL model.
Hence, SGFusion reduces better the discrepancy among fused gradients while providing sufficient knowledge from sampled zones to improve zone-FL models.

\section{Discussion}
The granularity of zones is an underexplored aspect in the current setting of SGFusion. 
Zones must not be too large or too small to encompass a sufficient number of users to maintain reliable and high model utility while reducing computational and operational costs. At the same time, the zones must capture localized behavioral differences, such as unique mobility patterns or resource consumption trends. One solution to this problem is to collect user mobility statistics at the edge nodes to identify common mobility patterns and then define the zones based on these patterns.

In addition, the edge infrastructure is another consideration for deploying SGFusion in real-world scenarios. In the current setting, each zone is associated with one edge device, which manages the mobile devices within its boundaries. Assuming the wireless telecom operators will deploy edge devices at scale in the near future, practical deployment of SGFusion may have several edge devices within one zone. A direct solution is to allow the mobile devices to communicate with any edge devices, but only one edge device must be responsible for the zone. The practical deployment will need communication protocols between the edge devices to ensure the correct functionality of SGFusion. We plan to explore these open research questions in our future work towards real-world deployment of SGFusion.

\section{Conclusions}
This paper presented \textbf{SGFusion}, a novel FL training algorithm for geographical zones, that models local data label distribution-based correlations among geographical zones as hierarchical and probabilistic random graphs, optimized by MCMC sampling. At each step, every zone samples a set of zones from its associated probabilistic dendrogram to fuse its local gradient with shared gradients from these zones. SGFusion enables knowledge fusion and sharing among zones in a probabilistic and stochastic gradient fusion process with self-attention weights. 
Theoretical and empirical results show that models trained with SGFusion converge with upper-bounded expected errors and remarkable better model utility without notable cost in system scalability compared with baselines.


\bibliographystyle{IEEEtran}
\bibliography{refTemp}

\end{document}